\documentclass{article}

\usepackage{iclr2027/iclr2027_conference,times}
\usepackage[english]{babel}
\usepackage[autostyle=true]{csquotes}
\usepackage{mathtools,amssymb,amsthm}
\usepackage{siunitx}
\usepackage{microtype}
\usepackage{thmtools}
\usepackage{thm-restate}
\usepackage{tikz}
\usepackage{booktabs}

\usepackage{url}
\usepackage{hyperref}
\usepackage[nameinlink,capitalise,noabbrev]{cleveref}

\hypersetup{hypertexnames=false}

\newcommand*{\RR}{\mathbb{R}}
\newcommand*{\QQ}{\mathbb{Q}}
\newcommand*{\ZZ}{\mathbb{Z}}
\newcommand*{\NN}{\mathbb{N}}
\newcommand*{\CPWL}{\mathrm{CPWL}}
\newcommand*{\symgrp}[1]{S_{#1}}
\newcommand*{\maxfn}[1]{\max\nolimits_{#1}}
\newcommand*{\sortedcone}{\mathcal{C}}
\newcommand*{\bigO}{\mathcal{O}}

\DeclarePairedDelimiter{\abs}{\lvert}{\rvert}
\DeclarePairedDelimiter{\floor}{\lfloor}{\rfloor}
\DeclarePairedDelimiter{\ceil}{\lceil}{\rceil}
\DeclarePairedDelimiter{\set}{\lbrace}{\rbrace}
\DeclarePairedDelimiterX{\scalarp}[2]{\langle}{\rangle}{#1, #2}
\DeclarePairedDelimiterX{\segment}[2]{[}{]}{#1,#2}
\DeclarePairedDelimiterX{\mset}[1]{\lbrace}{\rbrace}{%
  \!\delimsize\lbrace #1 \delimsize\rbrace\!%
}
\DeclarePairedDelimiterX{\rset}[2]{\lbrace}{\rbrace}{%
  #1%
  \nonscript\:\delimsize\vert
  \allowbreak
  \nonscript\:%
  \mathopen{}%
  #2%
}
\DeclarePairedDelimiterX{\rmset}[2]{\lbrace}{\rbrace}{%
  \!\delimsize\lbrace
  #1%
  \nonscript\:\delimsize\vert
  \allowbreak
  \nonscript\:%
  \mathopen{}#2%
  \delimsize\rbrace\!%
}

\DeclareMathOperator{\conv}{conv}
\DeclareMathOperator{\relu}{ReLU}
\DeclareMathOperator{\sym}{sym}
\DeclareMathOperator{\Span}{span}

\theoremstyle{plain}
\newtheorem{theorem}{Theorem}[section]
\newtheorem*{theorem*}{Theorem}
\newtheorem{proposition}[theorem]{Proposition}
\newtheorem{corollary}[theorem]{Corollary}
\newtheorem{lemma}[theorem]{Lemma}
\newtheorem{conjecture}[theorem]{Conjecture}
\theoremstyle{definition}

\theoremstyle{remark}
\newtheorem{remark}[theorem]{Remark}

\newcommand{\affilmark}[1]{\textsuperscript{\normalfont #1}}
\newlength{\authorblockwidth}
\title{Shallower ReLU Network Representations\\via Exact Linear Algebra}

\author{%
  \parbox[t]{\authorblockwidth}{%
    \centering
    \bfseries
    \begin{tabular}{@{}%
        c@{\hspace{1.15em}}%
        c@{\hspace{1.15em}}%
        c@{\hspace{1.15em}}%
      c@{}}
      Kilian Rueß\affilmark{1,}%
      \thanks{Lead author. All remaining authors are listed
      alphabetically by surname.}
      &
      Gennadiy Averkov\affilmark{2}
      &
      Florestan Brunck\affilmark{3}
      &
      Moritz Grillo\affilmark{4}
      \\
      Christoph Hertrich\affilmark{1}
      &
      Georg Loho\affilmark{5}
      &
      Jack Stade\affilmark{3}
      &
      Moritz Stargalla\affilmark{1}
      \\
      \multicolumn{4}{c}{%
        Matthew Sun\affilmark{6}%
        \hspace{1.3em}%
        Martin Winter\affilmark{4}%
      }
    \end{tabular}
    \\[7pt]
    \normalfont\footnotesize
    \affilmark{1}\,University of Technology Nuremberg
    \qquad
    \affilmark{2}\,Brandenburg University of Technology
    Cottbus--Senftenberg
    \\[1.5pt]
    \affilmark{3}\,University of Copenhagen
    \qquad
    \affilmark{4}\,Max Planck Institute for Mathematics in the Sciences
    \\[1.5pt]
    \affilmark{5}\,Freie Universität Berlin
    \qquad
    \affilmark{6}\,Amador Valley High School
    \\[6pt]
    \fontsize{8}{9.5}\selectfont
    \texttt{%
      \{kilian.ruess,christoph.hertrich,moritz.stargalla\}@utn.de%
    }
    \\[0.5pt]
    \begin{tabular}{@{}%
        c@{\hspace{1em}}%
        c@{\hspace{1em}}%
      c@{}}
      \texttt{averkov@b-tu.de}
      &
      \texttt{\{flbr,jast\}@di.ku.dk}
      &
      \texttt{\{moritz.grillo,martin.winter\}@mis.mpg.de}
    \end{tabular}
    \\[0.5pt]
    \texttt{georg.loho@math.fu-berlin.de}
    \hspace{1.4em}
    \texttt{mattsun1@mit.edu}%
  }%
}

\iclrfinalcopy

\begin{document}
\maketitle
\begin{abstract}
  We study the depth required by ReLU networks to exactly represent piecewise linear functions, focusing specifically on the maximum function. This problem has recently received significant attention in both the ML and TCS literature. We prove that \(\maxfn{n}(x)=\max\{x_1,\ldots,x_n\}\) is exactly representable with two hidden layers for every \(n\leq 12\). Previously, this was only known up to \(n\leq5\) \citep[STOC'26]{bakaevBetterNeuralNetwork2026}.
We obtain our constructions through an exact computer-assisted search within a space of candidate solutions: After a symmetry reduction, we obtain a finite system of linear equations over \(\QQ\) such that any solution yields a valid representation of the maximum function.

The resulting constructions have a structured first hidden layer, which enables recursive substitution into deeper networks. This yields an exact ReLU representation of \(\maxfn{n}\) with at most \(\ceil*{\log_6(n/2)}+1\) hidden layers.
Consequently, every continuous piecewise-linear function on \(\RR^d\) admits an exact representation with at most \(\ceil*{\log_6((d+1)/2)}+1\) hidden layers; in particular, two hidden layers suffice for \(d\leq 11\). Again, these results improve upon \cite[STOC'26]{bakaevBetterNeuralNetwork2026}, who proved analogous logarithmic bounds with base three.

\end{abstract}
\section{Introduction}
We study the minimum number of hidden layers required to represent continuous piecewise-linear functions exactly by ReLU networks of unrestricted width. A scalar-output ReLU network with \(\ell\) hidden layers consists of affine maps \(T_i \colon \RR^{n_{i-1}}\to\RR^{n_i}\), for \(i=1,\dots,\ell+1\), with \(n_{\ell+1}=1\), and computes
\[
  f = T_{\ell+1} \circ \relu \circ \cdots \circ \relu \circ\, T_1,
\]
where \(\relu(x)\coloneqq\max\set{0,x}\) is applied coordinate-wise after each of the first \(\ell\) affine maps. For a given function \(f\), we ask for the smallest possible \(\ell\) when the hidden-layer widths are unrestricted.

Every ReLU network computes a continuous piecewise-linear (CPWL) function: the input space can be partitioned into finitely many convex polyhedral regions on each of which the network computes an affine function; conversely, every CPWL function admits an exact ReLU representation~\citep{aroraUnderstandingDeepNeural2018}. Exactness is essential here: with unrestricted width, one hidden layer networks can approximate continuous functions arbitrarily well on compact sets~\citep{leshnoMultilayerFeedforwardNetworks1993}, and approximation of \(\maxfn{n}\) is studied in~\cite{safranHowManyNeurons2024}. Our question is instead how many layers are required for equality on the entire domain.

As shown in previous work \citep{hertrichLowerBoundsDepth2023}, the representation problem for general \(\CPWL\) functions can be reduced to that of representing maximum functions. Define
\[
  \maxfn{n}(x) \coloneqq \max\set{x_1,\dots,x_n},
\]
and let \(\CPWL_d\) denote the space of \(\CPWL\) functions \(\RR^d\to\RR\). The generalized hinging-hyperplane theorem~\citep{wangGeneralizationHingingHyperplanes2005} states that every \(f\in\CPWL_d\) can be written as
\begin{equation}\label{eq:Wang-Sun}
  f = \sum_{i=1}^s \sigma_i\maxfn{d+1}(A_i(x)),
\end{equation}
where \(\sigma_i\in\set{\pm1}\) and each \(A_i\colon\RR^d\to\RR^{d+1}\) is affine. Hence a hidden-layer upper bound for \(\maxfn{d+1}\) extends to every function in \(\CPWL_d\): the \(\maxfn{d+1}\) terms can be evaluated in parallel, and their outputs can be combined by the final affine map. Conversely, any lower bound on \(\maxfn{d+1}\) would imply the same lower bound on \(x\mapsto\max\set{0,x_1,\dots,x_d}\), yielding a worst-case lower bound over \(\CPWL_d\); see \cite{hertrichLowerBoundsDepth2023}. Maximum functions therefore provide both universal upper bounds and worst-case lower-bound witnesses for exact representation at unrestricted width.

\subsection{Depth bounds for exact maximum representations}

The elementary identity
\begin{equation}\label{eq:max_with_relu}
  \max\set{x_1,x_2}
  =
  \relu(x_1)-\relu(-x_1)+\relu(x_2-x_1)
\end{equation}
represents \(\maxfn{2}\) with one hidden layer. After duplicating inputs if necessary, arranging pairwise maxima in a balanced binary tree yields a representation of \(\maxfn{n}\) with \(\ceil{\log_2 n}\) hidden layers and, by Wang-Sun~\eqref{eq:Wang-Sun}, a \(\ceil{\log_2(d+1)}\)-hidden-layer representation of every function in \(\CPWL_d\)~\citep[ICLR'18]{aroraUnderstandingDeepNeural2018}.

Despite substantial prior work, no function is known to require more than two hidden layers for exact representation when width is unrestricted. \cite[NeurIPS'21 / SIDMA'23]{hertrichLowerBoundsDepth2023} connected the problem to polyhedral geometry and conjectured that the balanced binary-tree construction \(\ceil{\log_2 n}\) is optimal for \(\maxfn{n}\). \cite[ICLR'23]{haaseLowerBoundsDepth2023} proved the conjectured lower bound for networks with integral weights. For decimal fractions~\cite[ICLR'25]{averkovExpressivenessRationalReLU2025} showed that representing \(\max\set{0,x_1,\ldots,x_n}\) requires at least \(\ceil*{\log_3(n+1)}\) hidden layers. Other lower bounds assume compatibility of the breakpoint set with the braid-fan~\citep[NeurIPS'25]{grilloDepthBoundsNeuralNetworks2025}, focus on monotone or input-convex networks \cite{bakaevDepthMonotoneReLU2025,valerdiMinimalDepthNeural2026}, or impose a width bound when the number of hidden layers is fixed~\citep[COLT'26]{safranDepthHierarchyComputing2026}.

Without additional assumptions on the weights or breakpoints, however, the binary-tree construction is not optimal: \cite[STOC'26]{bakaevBetterNeuralNetwork2026} disproved the conjecture of \cite{hertrichLowerBoundsDepth2023} by constructing a two-hidden-layer representation of \(\maxfn{5}\), and then bootstrapped this construction to show that \(\ceil{\log_3(n-2)}+1\) hidden layers suffice for \(\maxfn{n}\), \(n\ge4\). Through the Wang-Sun decomposition~\eqref{eq:Wang-Sun}, this yields \(\ceil{\log_3(d-1)}+1\) hidden layers for every function in \(\CPWL_d\), \(d\ge3\). Thus, prior to the present work, the best known unrestricted upper bound has base-\(3\) logarithmic scaling, matching the lower bound known under the more restrictive assumption of decimal-fraction weights.

\subsection{Main results}

Our first result is an exact two-hidden-layer representation of \(\maxfn{12}\) with rational weights, extending the largest previously known arity from \(5\) to \(12\)~\citep{bakaevBetterNeuralNetwork2026}. Duplicating input coordinates gives the same upper bound for every smaller arity; together with the known one-hidden-layer lower bound~\citep[Proposition~2.2]{mukherjeeLowerBoundsBoolean2017}, this determines the minimum number of hidden layers for every \(3\le n\le12\).

\begin{restatable}{theorem}{maxtwelvethm}
  \label{thm:max-12}
  For every integer \(3\le n\le12\), two hidden layers are necessary and sufficient to represent \(\maxfn{n}\) exactly by a ReLU network. Moreover, such a representation exists with rational weights.
\end{restatable}

The solutions have a highly structured first hidden layer, which conceptually computes only pairwise maxima. This structure enables an efficient recursive construction. Similarly to the construction of \citet{bakaevBetterNeuralNetwork2026}, which recursively reuses lower-arity maximum representations, we substitute such representations into the first hidden layer. In our construction, each additional hidden layer increases the represented arity by a factor of six; see \Cref{sec:bootstrapping}. This strengthens the resulting base-\(3\) bound of \citet{bakaevBetterNeuralNetwork2026} to base \(6\).

\begin{restatable}{corollary}{bootstrapcor}
  \label{cor:base-six}
  For every integer \(n\ge1\), the function \(\maxfn{n}\) has an exact ReLU representation with at most
  \[
    \ceil*{\log_6\! \left(\frac n2\right)} + 1
  \]
  hidden layers.
\end{restatable}

\Cref{tab:depth-comparison} compares the arities obtained at small depths by the binary tree~\citep{aroraUnderstandingDeepNeural2018}, the previous base-\(3\) construction~\citep{bakaevBetterNeuralNetwork2026}, and our construction.
\begin{table}[t]
  \centering
  \small
  \begin{tabular}{@{}cccc@{}}
    \toprule
    Hidden layers & \cite{aroraUnderstandingDeepNeural2018} & \cite{bakaevBetterNeuralNetwork2026} & This paper \\
    \midrule
    2 & \(\maxfn{4}\)  & \(\maxfn{5}\)  & \(\maxfn{12}\) \\
    3 & \(\maxfn{8}\)  & \(\maxfn{11}\) & \(\maxfn{72}\) \\
    4 & \(\maxfn{16}\) & \(\maxfn{29}\) & \(\maxfn{432}\) \\
    \bottomrule
  \end{tabular}
  \caption{Maximum arity representable at each depth by the balanced binary-tree construction~\citep{aroraUnderstandingDeepNeural2018}, the previous base-\(3\) construction~\citep{bakaevBetterNeuralNetwork2026}, and our base-\(6\) construction.}
  \label{tab:depth-comparison}
\end{table}
Equivalently, if the number of hidden layers is written as \(c\log n+O(1)\), the leading coefficient is reduced by a factor of \(\log(3)/\log(6)\approx0.613\).

Substituting \(n=d+1\) into the maximum-function bound and applying the Wang-Sun decomposition~\eqref{eq:Wang-Sun} gives the corresponding improvement for arbitrary \(\CPWL\) functions.

\begin{restatable}{corollary}{cpwlcor}
  \label{cor:cpwl}
  For \(d\ge1\), every function in \(\CPWL_d\) is exactly representable by a ReLU network with at most
  \[
    \ceil*{\log_6\! \left(\frac{d+1}{2}\right)} + 1
  \]
  hidden layers. In particular, every function in \(\CPWL_d\) with \(d\le11\) admits an exact representation with two hidden layers.
\end{restatable}

\paragraph{Conceptual contribution.}
While \citet{bakaevBetterNeuralNetwork2026} disproved the conjecture of \cite{hertrichLowerBoundsDepth2023} by pushing 2-hidden-layer representability from \(\max_4\) to \(\max_5\) using a clean geometric intuition, our conceptual contribution is that this can be pushed much further, to \(\max_{12}\), by selecting a suitable candidate space, performing symmetry reductions, and applying exact-arithmetic linear algebra computations.

Furthermore, the upper bound of \cite{bakaevBetterNeuralNetwork2026} almost matches the lower bound of \cite{averkovExpressivenessRationalReLU2025} for decimal fractions, both scaling logarithmically with base three. This leaves open whether general rational weights have an advantage over decimal fractions.
The rationality of our construction implies that this is indeed the case, improving the upper bound from base three to base six, and separating those two weight classes already for two hidden layers: for \(10\le n\le12\), rational weights suffice to represent \(\maxfn{n}\) with two hidden layers, whereas decimal-fraction weights do not; see \Cref{cor:weight-separation}.

Finally, at this point, it seems unlikely to us that our bounds are optimal, with computational limits being the main obstacle for scaling our approach.
The computational evidence provided in this paper therefore supports the following conjecture, which is the complementary extremal case compared to the now disproven conjecture by \citet{hertrichLowerBoundsDepth2023}.

\begin{conjecture}
    Every function in \(\CPWL_d\) in any dimension \(d\) can be exactly represented by a ReLU network with two hidden layers.
\end{conjecture}

As discussed above, to prove this, it would suffice to prove it for the $\max$ function.
It is of course conceivable that additional algorithmic insights will make it possible to compute \(\max_n\) with some \(n>12\) in two hidden layers, and every such improvement will consequently lead to improvements in \Cref{cor:base-six} and \Cref{cor:cpwl}. However, computational results of this type would only increase the base of the logarithm further. A genuine advancement at this point would be a sub-logarithmic upper bound or a lower bound of more than two hidden layers, both of which seem out of reach with current techniques and probably require nontrivial new theoretical insights.

\subsection{Idea of the construction}

Our computational search was inspired by the polyhedral interpretation of the \(\maxfn{5}\) construction of \cite{bakaevBetterNeuralNetwork2026}. 
Using the duality between convex, positively homogeneous CPWL functions and polytopes via support functions, the construction derives a formal Minkowski difference from a subdivision of the simplex, and hence a ReLU representation. 

A priori, the search for a representation of $\maxfn{n}$ takes place in an infinite-dimensional function space, and hence, it is not tractable using basic computational methods. 
To cast it into a finite-dimensional problem, we fix a finite family of candidate polytopes as generators for formal Minkowski differerences. 
In \Cref{sec:linear-system}, we describe the general scheme that then, finding coefficients in such a formal Minkowski difference is equivalent to solving a finite linear system. 
Note that we make a specific choice of candidate polytopes for which a solution is not guaranteed even if there may be a solution to the actual problem. 
We aimed for particularly well-structured families of polytopes to make the search space as small as possible as described below. 


For the actual search, we work directly with the equivalent piecewise-linear functions and use rank-\(2\) maxout networks. 
A rank-\(2\) maxout unit~\citep{goodfellowMaxoutNetworks2013} has the form
\[
  x \longmapsto \max\set{L_1(x),L_2(x)},
\]
where \(L_1,L_2\) are affine. A rank-\(2\) maxout network is a neural network in which every hidden unit is a rank-\(2\) maxout unit. The following elementary observation shows that passing to this representation does not effect the number of hidden layers. It follows immediately by applying \eqref{eq:max_with_relu} to each maxout unit.

\begin{lemma}
  \label{lem:maxout-to-relu}
  Every rank-\(2\) maxout network can be converted into a ReLU network with the same number of hidden layers. The conversion increases the width of each hidden layer by at most a factor of three; if the maxout network has rational weights, so does the resulting ReLU network.
\end{lemma}

While the first ingredient of our approach involved a choice, the following second ingredient is applicable in general to reduce the size of the search space further. 
The key observation is that $\maxfn{n}$ is symmetric in its $n$ arguments, so it suffices to look for permutation-invariant solutions.  
We average each candidate over coordinate permutations, so that both the target and the resulting candidates are \(\symgrp{n}\)-invariant, where \(\symgrp{n}\) denotes the symmetric group on \(n\) elements. Every vector in \(\RR^n\) can be permuted into the cone of sorted vectors
\[
  \rset*{x\in\RR^n}{x_1\le\cdots\le x_n}.
\]
Since both sides of the desired equality are permutation-invariant, it therefore does suffice to establish equality on this cone. There, \(\maxfn{n}(x)=x_n\), and the pairwise maxima in the first hidden layer of our construction simplify to coordinate projections. Expanding the remaining maxima into linear and nonlinear ReLU terms reduces the desired identity to a finite linear system over \(\QQ\): the nonlinear terms must cancel, and the remaining linear term must equal \(x_n\). We solve these systems exactly, obtain rational rank-\(2\) maxout representations, and convert them to ReLU representations of the same depth using \Cref{lem:maxout-to-relu}. The resulting identities are checked independently by an exact-arithmetic verifier.

\section{Symmetry reduction to the sorted cone}
\label{sec:symmetrization}

Permutation symmetry allows us to reduce checking whether two symmetric functions agree on \(\RR^n\) to checking whether they agree on the cone of sorted vectors
\[
  \sortedcone \coloneqq \rset{x \in \RR^n}{x_1 \leq \dots \leq x_n}.
\]
Let the symmetric group \(\symgrp{n}\) act on \(\RR^n\) by permuting coordinates. Explicitly, for \(x \in \RR^n\) and \(\sigma \in \symgrp{n}\), define
\[
  (\sigma x)_i \coloneqq x_{\sigma^{-1}(i)} \qquad \text{for } i \in \set{1, \ldots, n}.
\]
For any function \(f \colon \RR^n\to\RR\), define its \emph{symmetrization} by
\[
  f^{\sym}(x)\coloneqq\frac1{n!}\sum_{\sigma\in\symgrp{n}} f(\sigma x).
\]
We call \(f \colon \RR^n \to \RR\) \emph{symmetric} if \(f(\sigma x)=f(x)\) for every \(x\in\RR^n\) and \(\sigma\in\symgrp{n}\). The symmetrization operator is linear and maps any function to a symmetric function. In particular, \(f\) is symmetric if and only if \(f=f^{\sym}\). In invariant theory, this symmetrization is called the \emph{averaging} (or \emph{Reynolds}) operator associated with the action of \(\symgrp{n}\); see, e.g., \citet{derksen2015computational}.

\begin{proposition}
  \label{prop:symmetric-reduction}
  Let \(g \colon \RR^n\to\RR\) be symmetric, and suppose that
  \(g=\sum_{i=1}^m c_i f_i\) for some functions \(f_i \colon \RR^n \to \RR\) and coefficients \(c_i \in \RR\). Then
  \[
    g=\sum_{i=1}^m c_i f_i^{\sym}.
  \]
  Moreover, if \(f,h \colon \RR^n\to\RR\) are symmetric, then \(f=h\) on \(\RR^n\) if and only if \(f=h\) on \(\sortedcone\).
\end{proposition}

\begin{proof}
  By linearity of symmetrization and symmetry of \(g\),
  \[
    g
    = g^{\sym}
    = \left(\sum_{i=1}^m c_i f_i\right)^{\sym}
    = \sum_{i=1}^m c_i f_i^{\sym},
  \]
  which proves the first claim. For the second claim, only the reverse implication requires proof. Suppose that \(f=h\) on \(\sortedcone\), and let \(x\in\RR^n\). Choose \(\sigma\in\symgrp{n}\) such that \(\sigma x\in\sortedcone\). By symmetry
  \[
    f(x)
    =
    f(\sigma x)
    =
    h(\sigma x)
    =
    h(x).
  \]
  Hence \(f=h\) on \(\RR^n\).
\end{proof}

Since \(\maxfn{n}(x) = x_n\) for every \(x \in \sortedcone\), a symmetric function equals \(\maxfn{n}\) if and only if it equals \(x_n\) on \(\sortedcone\). The cone \(\sortedcone\) is the closure of a chamber of the braid arrangement \(\rset{x_i=x_j}{1\leq i<j\leq n}\). This braid arrangement has been previously used to study exact representations of the maximum function by ReLU networks; see~\cite{grilloDepthBoundsNeuralNetworks2025}.

\section{Exact search for two-hidden-layer identities}
\label{sec:computations}

We now turn the symmetry reduction into a finite exact search. We first specify a restricted family of two-hidden-layer rank-\(2\) maxout blocks whose first hidden layer consists only of pairwise maxima. We then symmetrize these blocks, restrict them to the sorted cone, and express the condition that a rational linear combination equals \(\maxfn{n}\) as a finite linear system over \(\QQ\).

\subsection{The two-hidden-layer ansatz}
\label{sec:ansatz}
Let $\mathcal{E}_n
  \coloneqq
  \rset{(i,j)}{1\le i\le j\le n}
$ denote the unordered pairs of indices with repetition, including the diagonal pairs, and define
\[
  m_{ij}(x)\coloneqq\max\set{x_i,x_j} \quad \text{ for } x \in \RR^n .
\]
For \(k\in\NN\), let \(\mathcal{M}_{n,k}\) be the set of multisets of cardinality \(k\) with elements in \(\mathcal{E}_n\). For \(A,B\in\mathcal{M}_{n,k}\), define
\[
  \Phi_{A,B}(x)
  \coloneqq
  \max\set*{
    \sum_{(i,j)\in A}m_{ij}(x),
    \sum_{(i,j)\in B}m_{ij}(x)
  },
\]
where, since both \(A\) and \(B\) have cardinality \(k\), each sum consists of exactly \(k\) terms. This is a two-hidden-layer rank-\(2\) maxout block: the first hidden layer computes pairwise maxima, and the second computes the maximum of two sums of~\(k\) first-layer outputs. By \Cref{lem:maxout-to-relu}, every such block has a two-hidden-layer ReLU realization with rational weights. The parameter \(k\) controls the complexity of the candidate blocks and the size of the resulting search space. To impose permutation symmetry, define
\[
  F_{A,B}(x)
  \coloneqq
  \sum_{\sigma\in\symgrp{n}}\Phi_{A,B}(\sigma x)
  =
  n!\,\Phi_{A,B}^{\sym}(x).
\]
The factor \(1/n!\) from the symmetrization operator is omitted because it only rescales the coefficients of a representation. We consider the vector space spanned by the finitely many \(F_{A, B}\)
\[
  V_{n,k}
  \coloneqq
  \Span\rset*{F_{A,B}}{A,B\in\mathcal{M}_{n,k}}.
\]
Equivalently, \(V_{n,k}\) is the space of all functions that can be represented with a ReLU network using the prescribed first two hidden layers, by choosing output layer weights. Our goal is therefore to find a certificate that \(\maxfn{n}\in V_{n,k}\).
The spaces are nested, \(V_{n,k}\subseteq V_{n,k+1}\), as shown in \Cref{lem:span-nested}. Moreover, \(\maxfn{n} \notin V_{n,k}\) for every \(k<\floor*{(n-1)/2}\) by the polyhedral dimension obstruction in \Cref{cor:min-k}.

\subsection{Reduction to an exact linear system}
\label{sec:exact-linear-system}

By \Cref{prop:symmetric-reduction}, it suffices to establish the desired identity on the sorted cone \(\sortedcone\).
For \(x \in \sortedcone\), the coordinates are ordered, so the maximum of two coordinates is the one with the larger index. Hence
\[
  m_{ij}(\sigma x)
  =
  x_{\max\set{\sigma^{-1}(i),\sigma^{-1}(j)}}.
\]
Accordingly, for \(A\in\mathcal{M}_{n,k}\) and \(\sigma\in\symgrp{n}\) define the linear form
\[
  \ell_{\sigma,A}(x)
  \coloneqq
  \sum_{(i,j)\in A}
  x_{\max\set{\sigma^{-1}(i),\sigma^{-1}(j)}}.
\]
Then, for \(x\in\sortedcone\),
\[
  F_{A,B}(x)
  =
  \sum_{\sigma\in\symgrp{n}}
  \max\set{\ell_{\sigma,A}(x),\ell_{\sigma,B}(x)}.
\]

Using
\[
  \max\set{u,v}
  =
  u+\relu(v-u),
\]
and grouping equal difference vectors, we can write
\[
  F_{A,B}(x)
  =
  L_{A,B}(x)
  +
  \sum_{d\in D_{A,B}}
  c_{A,B,d}\,
  \relu(d^\top x),
\]
where \(L_{A,B}\) is linear, \(D_{A,B}\subset\ZZ^n\) is finite, and \(c_{A,B,d}\in\ZZ\). We search for rational coefficients \(\lambda_{A,B}\) satisfying
\begin{align}
  \sum_{A,B}\lambda_{A,B}c_{A,B,d}
  &=0
  &&\text{for every }d\in\bigcup_{A,B}D_{A,B},
  \label{eq:nonlinear-cancel}\\
  \sum_{A,B}\lambda_{A,B}L_{A,B}(x)
  &=x_n.
  \label{eq:linear-xn}
\end{align}
The equations in \eqref{eq:nonlinear-cancel} cancel the nonlinear ReLU terms, while \eqref{eq:linear-xn} requires the remaining linear part to equal \(x_n=\maxfn{n}(x)\) on \(\sortedcone\). Thus the search reduces to a finite linear system over \(\QQ\). Any rational solution \((\lambda_{A,B})\) gives an exact identity
\[
  \maxfn{n}
  =
  \sum_{A,B}\lambda_{A,B}F_{A,B}
\]
on \(\sortedcone\), and hence on all of \(\RR^n\) by permutation symmetry. Since the blocks \(F_{A,B}\) have rank-\(2\) maxout realizations with rational weights and the coefficients \(\lambda_{A,B}\) are rational, the resulting network also has rational weights. Applying \Cref{lem:maxout-to-relu} then gives a ReLU network of the same depth with rational weights.

\begin{remark}\label{rem:linear-deps}
  Before assembling the system, we simplify the ReLU terms using elementary identities valid on \(\sortedcone\). If \(d^\top x\) has fixed sign on \(\sortedcone\), then \(\relu(d^\top x)\) is either identically zero or linear on the cone and can therefore be absorbed into the linear part. Opposite directions are related by
  \[
    \relu(-d^\top x)
    =
    \relu(d^\top x)-d^\top x,
  \]
  and positive homogeneity gives
  \[
    \relu(q\,d^\top x)
    =
    q\,\relu(d^\top x),
    \qquad q>0.
  \]
  We use these identities to merge equivalent nonlinear terms, adjusting their coefficients and the linear part accordingly.
\end{remark}

\subsection{Enumeration modulo symmetries}

The candidate pairs \((A,B)\) have two immediate redundancies. First, \(\Phi_{A,B}=\Phi_{B,A}\).  so \(F_{A,B}=F_{B,A}\). Second, simultaneous relabelling of all indices merely permutes the summands in the orbit sum. Hence
\[
  F_{A,B}
  =
  F_{B,A}
  =
  F_{\tau A,\tau B}
  =
  F_{\tau B,\tau A}
  \qquad
  \text{for every }\tau\in\symgrp{n},
\]
where \(\tau A\) is obtained by applying \(\tau\) to every index and sorting the two entries of each resulting pair.

Let \(C_2\cong\ZZ/2\ZZ=\set{0,1}\) act by exchanging the two multisets, and define the action of \(\symgrp{n}\times C_2\) by
\[
  (\tau,\varepsilon)\cdot(A,B)
  \coloneqq
  \begin{cases}
    (\tau A,\tau B), & \varepsilon=0,\\
    (\tau B,\tau A), & \varepsilon=1.
  \end{cases}
\]
Since \(F_{A,B}\) is constant on every orbit, only one representative from each orbit is needed in the linear system.

Equivalently, \(A\) and \(B\) may be viewed as the two edge colors of a multigraph on \(\set{1,\ldots,n}\), with loops and edge multiplicities allowed. Permuting the coordinates relabels the vertices, while exchanging \(A\) and \(B\) exchanges the two colors. Orbit enumeration therefore reduces to a colored-multigraph isomorphism problem with interchangeable colors. We enumerate these equivalence classes using \emph{nauty}~\citep{mckayPracticalGraphIsomorphism2014}.

\begin{table}[t]
  \centering
  \begin{tabular}{@{}c l l@{}}
    \toprule
    \(i\) & \(A_i\)                & \(B_i\)                \\
    \midrule
    1     & \(\mset{(1,2),(1,2)}\) & \(\mset{(3,4),(3,4)}\) \\
    2     & \(\mset{(1,2),(1,3)}\) & \(\mset{(1,4),(5,6)}\) \\
    3     & \(\mset{(1,2),(1,3)}\) & \(\mset{(4,5),(4,6)}\) \\
    4     & \(\mset{(1,2),(3,4)}\) & \(\mset{(1,3),(5,6)}\) \\
    \bottomrule
  \end{tabular}
  \caption{The four orbit representatives used in the certificate for \(\maxfn{6}\).}
  \label{tab:certificate-six}
\end{table}

\subsection{Computational results}
\maxtwelvethm*

\begin{proof}
  The cases \(n=3,4\) follow from balanced binary-tree rank-\(2\) maxout representations. For each \(5\le n\le12\), we use \(k=\floor*{(n-1)/2}\), construct the system described above, apply the reductions from \Cref{rem:linear-deps}, and solve it exactly over \(\QQ\). The system is solvable, and the coefficients \(\lambda_{A,B}\) of the resulting solution satisfy \eqref{eq:nonlinear-cancel} and \eqref{eq:linear-xn}, and hence yield
  \[
    \maxfn{n}
    =
    \sum_{A,B}\lambda_{A,B}F_{A,B}
  \]
  on \(\sortedcone\). Since both sides are permutation-invariant, \Cref{prop:symmetric-reduction} extends the identity to all of \(\RR^n\).

  Each \(F_{A,B}\) is realized by evaluating two-hidden-layer rank-\(2\) maxout blocks in parallel and combining their outputs linearly. All coefficients appearing inside the blocks are integral, and the coefficients \(\lambda_{A,B}\) are rational, so the resulting maxout network has rational weights. By \Cref{lem:maxout-to-relu}, it converts to a ReLU network with the same number of hidden layers and rational weights. The resulting identities are checked independently by the exact-arithmetic verifier.

  Finally, \cite[Proposition~2.2]{mukherjeeLowerBoundsBoolean2017} proves that \(\max\set{0,x_1,x_2}\) has no one-hidden-layer ReLU representation. Restricting \(\maxfn{n}\) to inputs of the form \((x_1,x_2,0,\ldots,0)\) transfers this lower bound to every \(n\ge3\), proving optimality.
\end{proof}

\Cref{tab:computational-statistics} summarizes the search for \(n\le12\). Here \(N_{\mathrm{rep}}\) denotes the number of \((A,B)\)-orbits retained after the \(\symgrp{n}\times C_2\) symmetry reduction, \(D\) is the set of direction vectors indexing nonlinear ReLU terms after the reductions in \Cref{rem:linear-deps}, and the support size is the number of nonzero coefficients \(\lambda_{A,B}\) in the reported certificate.

\begin{remark}
  In our computations, we observe that the ansatz can be restricted considerably further while the resulting systems of linear equations remain solvable. In particular, it suffices to exclude diagonal pairs \((i, i)\) and to take \(A\) and \(B\) to be sets of cardinality \(k\) rather than multisets, so that no pairs occurs with multiplicity. Moreover, we can restrict to pairs \((A, B)\) such that \(A \cap B = \varnothing\) and such that, when \(A\) and \(B\) are viewed separately as edge sets of graphs on \(\set{1, \dots, n}\), both graphs are forests, i.e., acyclic.
\end{remark}

\begin{table}[t]
  \centering
  \begin{tabular}{
      cc
      S[table-format=8.0]
      S[table-format=8.0]
      S[table-format=7.0]
      S[table-format=5.0]
    }
    \toprule
    \(n\) & \(k\) &
    {\(\abs{\mathcal{M}_{n,k}}\)} &
    {\(N_{\mathrm{rep}}\)} &
    {\(\abs{D}\)} &
    {Support size} \\
    \midrule
    5  & 2 &    120 &    131 &    15 &   3 \\
    6  & 2 &    231 &    144 &    35 &   4 \\
    7  & 3 &   4060 &   4469 &   630 &  57 \\
    8  & 3 &   8436 &   4716 &  1428 &  69 \\
    9  & 4 & 194580 & 210540 & 20685 & 337 \\
    10 & 4 & 424270 & 216428 & 47157 & 402 \\
    11 & 5 & 12103014 & 12179657 & 657822 & 39042 \\
    12 & 5 & 27285336 & 12337637 & 1504932 & 16536 \\
    \bottomrule
  \end{tabular}
  \caption{Computational statistics for \(n\le12\). Here \(\abs{\mathcal{M}_{n,k}}\) is the number of individual size-\(k\) multisets, \(N_{\mathrm{rep}}\) is the number of orbit representatives of pairs \((A,B)\), and \(\abs{D}\) is the number of nonlinear direction vectors remaining after the reductions in \Cref{rem:linear-deps}.
  }
  \label{tab:computational-statistics}
\end{table}

\begin{corollary}\label{cor:weight-separation}
  For \(10\le n\le12\), \(\maxfn{n}\) admits a two-hidden-layer ReLU representation with rational weights, whereas no two-hidden-layer ReLU representation with decimal-fraction weights exists.
\end{corollary}

\begin{proof}
  The rational representation follows from \Cref{thm:max-12}. If a two-hidden-layer network with decimal-fraction weights represented \(\maxfn{n}\), then setting one input identically to zero would yield a two-hidden-layer representation of
  \[
    \max\set{0,x_1,\ldots,x_{n-1}}
  \]
  with decimal-fraction weights. For \(10\le n\le12\), this contradicts the lower bound \(\ceil*{\log_3 n}=3\) of \cite{averkovExpressivenessRationalReLU2025}.
\end{proof}

For \(n=6\) and \(k=2\), \Cref{tab:certificate-six} lists four pairs of multisets that give the certificate
\[
  \max\set{x_1,\ldots,x_6}
  =
  \frac{1}{720}F_{A_1,B_1}
  +\frac{1}{360}F_{A_2,B_2}
  -\frac{1}{1440}F_{A_3,B_3}
  -\frac{1}{360}F_{A_4,B_4},
\]
where \(F_{A,B}\) is the symmetrized block defined in \Cref{sec:ansatz}. The certificates for \(n=7,\ldots,12\) are substantially larger and are therefore not reproduced in the body of the paper.
\ificlrfinal
They are provided with the accompanying online material\footnote{\url{https://github.com/kilianar/max-relu-certificates}}.
\else
They are provided with the anonymized supplementary material.
\fi

The search is designed to optimize depth rather than width. The direct realization obtained from the symmetrized identities can have factorial second-layer width; \Cref{sec:construction-width} discusses this width and compares it with the known quadratic lower bound~\citep{safranDepthHierarchyComputing2026}.

\section{Structured substitution and the base-six hidden-layer bound}
\label{sec:bootstrapping}

To obtain representations of \(\maxfn{n}\) for larger \(n\), any \(D\)-hidden-layer representation of \(\maxfn{l}\) can be composed in a tree-like fashion: applying copies of the representation to groups of at most \(l\) inputs increases the arity by a factor of \(l\) at the cost of \(D\) additional hidden layers.

For the specific two-hidden-layer representation constructed in this paper, the additional structure of its first hidden layer allows a more efficient composition. Exploiting this structure, each step increases the arity by a factor of \(r \coloneqq\floor*{l/2}\) while requiring only one additional hidden layer.

\begin{theorem}
  \label{thm:structured-bootstrap}
  Let \(\mathcal{N}_l\) be a two-hidden-layer rank-\(2\) maxout network representing \(\maxfn{l}\), whose first hidden layer consists only of pairwise maxima \(\max\set{x_i,x_j}\), and let \(r\coloneqq\floor*{l/2}\ge2\).
  Then, for every integer \(s\ge0\), the function \(\maxfn{n}\) for \(n=lr^s\) has an exact rank-\(2\) maxout representation with \(2+s\) hidden layers whose first hidden layer again consists only of pairwise maxima. Consequently, for every \(n\ge1\), \(\maxfn{n}\) has an exact ReLU representation with at most
  \[
    2+\max\set*{0,\ceil*{\log_r\frac nl}}
  \]
  hidden layers. If \(\mathcal N_l\) has rational weights, then all of these representations can be chosen with rational weights.
\end{theorem}

\begin{proof}
  We induct on \(s\), maintaining the structure of the first hidden layer. For \(s=0\), the required representation is \(\mathcal N_l\).

  Suppose the claim holds for some \(s\ge0\), and set \(M\coloneqq lr^s\).  Partition \(rM\) input coordinates into \(M\) disjoint blocks \(S_1,\ldots,S_M\), each of size \(r\), and write
  \[
    y_p\coloneqq\max_{i\in S_p}x_i.
  \]
  Conceptually, evaluating the representation of \(\maxfn{M}\) at
  \((y_1,\ldots,y_M)\) gives
  \[
    \max_{1\le p\le M}y_p
    =
    \max_{1\le i\le rM}x_i.
  \]
  By the induction hypothesis, every unit in the first hidden layer of this representation is a pairwise maximum. After the substitution \(x_p\mapsto y_p\), such a unit becomes
  \[
    \max\set{y_p,y_q}
    =
    \max_{i\in S_p\cup S_q}x_i.
  \]
  This is a maximum of at most \(2r\le l\) original input coordinates. Hence it can be computed by a copy of \(\mathcal N_l\), with repeated-coordinate padding if fewer than \(l\) distinct inputs occur.

  In the network \(\mathcal{N}_M\) representing \(\max_M\), replace every first layer unit in parallel by a copy of \(\mathcal N_l\). The output affine maps of these copies can be absorbed into the affine map feeding the next maxout layer. Thus one hidden layer is replaced by two, increasing the total depth from \(2+s\) to \(3+s\). The resulting network represents
  \[
    \maxfn{rM}
    =
    \maxfn{lr^{s+1}}.
  \]
  Moreover, its first hidden layer consists of the pairwise-max units from the inserted copies of \(\mathcal N_l\), so the induction invariant is preserved.

  For arbitrary \(n\ge1\), set
  \[
    s
    \coloneqq
    \max\set*{0,\ceil*{\log_r\frac nl}}.
  \]
  Then \(lr^s\ge n\). Feeding the \(n\) input coordinates into the \(lr^s\) inputs and repeating coordinates as necessary gives a representation of \(\maxfn{n}\) with the same number of hidden layers. Finally, \Cref{lem:maxout-to-relu} converts the resulting rank-\(2\) maxout network into a ReLU network without increasing its depth. If the weights of \(\mathcal N_l\) are rational, the substitutions, affine-map absorptions, repeated-coordinate padding, and maxout-to-ReLU conversion all preserve rationality.
\end{proof}

\begin{figure}
  \centering
  \resizebox{0.95\linewidth}{!}{%
    \input{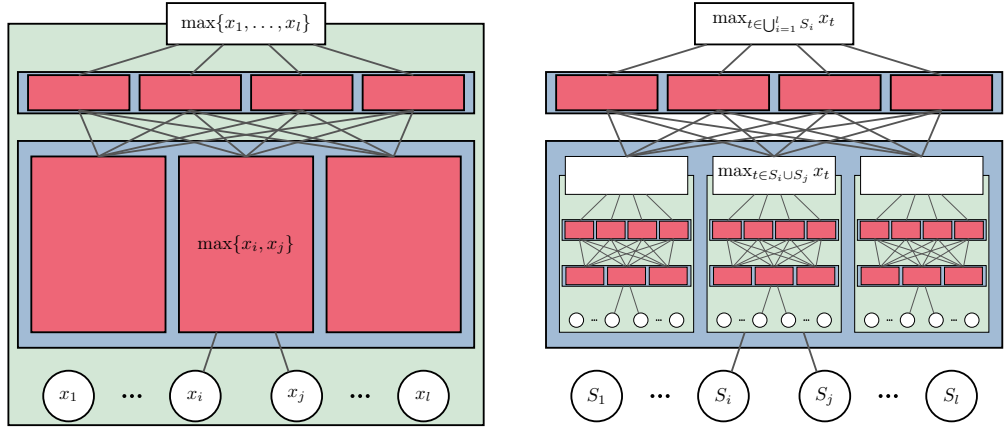}
  }
  \caption{Structured substitution: replacing each comparison of two block maxima by \(\mathcal N_l\) transforms a depth-\(2\) representation of \(\max_{lr}\) into a depth-\(3\) representation of \(\max_{lr^{2}}\).}
  \label{fig:bootstrap}
\end{figure}

\bootstrapcor*

\begin{proof}
  Apply \Cref{thm:structured-bootstrap} to the two-hidden-layer rank-\(2\) maxout representation of \(\maxfn{12}\) from \Cref{thm:max-12}. Here \(l=12\) and \(r=6\), so for every \(n\ge3\) the theorem gives
  \[
    2+\max\set*{0,\ceil*{\log_6\!\left(\frac{n}{12}\right)}}
    =
    2+\ceil*{\log_6\!\left(\frac{n}{12}\right)}
    =
    \ceil*{\log_6\!\left(\frac n2\right)}+1.
  \]
  The first equality holds because \(n\ge3\) implies
  \(\ceil*{\log_6(n/12)}\ge0\).
\end{proof}

\cpwlcor*

\begin{proof}
  By the Wang-Sun decomposition~\eqref{eq:Wang-Sun}, each \(f\in\CPWL_d\) is a signed sum of ReLU networks for \(\maxfn{d+1}\) composed with affine maps; parallelization and summation in the final affine layer preserve depth. Applying the bound for \(\maxfn{d+1}\) from  \Cref{cor:base-six} shows that \(f\) admits an exact ReLU representation with at most \(\ceil*{\log_6\!\left(\frac{d+1}{2}\right)}+1\) hidden layers.
\end{proof}

\section{Conclusion and Open Problems}
The two-hidden-layer representation of \(\maxfn{5}\) showed that the binary-tree depth bound is not optimal, but it remained unclear whether this was a low-arity exception. Our results rule this out: two hidden layers suffice for \(\maxfn{n}\) for every \(n\leq 12\), and hence for every \(\CPWL\) function on \(\RR^d\) with \(d\leq 11\). Thus, the range in which arbitrary \(\CPWL\) functions admit exact two-hidden-layer ReLU representations is considerably larger than previously known. Whether two hidden layers suffice in all dimensions remains open.

The main computational step is to replace a direct nonlinear search over network parameters by exact linear algebra. Within a restricted rank-\(2\) maxout ansatz, permutation symmetry and restriction to the sorted cone reduce the representation problem to a rational linear system. This makes larger arities accessible, but the system size grows rapidly. In particular, \(n=13\) remains unresolved because the corresponding computation exceeds our available resources, not because the ansatz is known to fail. There is currently no indication that the solutions stop at \(n=12\).

At the same time, the ansatz is highly restrictive. If there is no solution for some arity, this would exclude only this particular construction, not a general two-hidden-layer ReLU representation of \(\maxfn{n}\). A more useful next step is therefore to understand the structure behind the observed solutions. A structural description of the corresponding signed Minkowski relations could replace the arity-by-arity linear systems by a general argument and thereby extend the construction beyond the computational range.

Symmetrization is crucial for the present computation, but it also produces very large second hidden layers and may obscure the structure of the identities. Sparse or nonsymmetric solutions could reduce this width and, more importantly, reveal patterns hidden by full symmetrization.

Finally, the pairwise-max structure of our constructions yields the substitution in \Cref{thm:structured-bootstrap}, improving the general recursion from base \(3\) to base \(6\). Larger fixed-arity solutions of the same form can improve this base further, but they still give only \(\Theta(\log n)\) depth. Any \(o(\log n)\) bound therefore requires a stronger recursion whose gain grows with arity, or a different structural principle relating solutions across arities.

\paragraph{Reproducibility Statement.}
All computer-assisted claims in this work can be independently verified using the artifacts provided in the supplementary material. For \(5\leq n\leq12\), we provide the certificates used in \cref{sec:computations} together with a standalone Python verifier that checks the relevant identities using exact rational arithmetic. The verifier is independent of the more complex search code: verifying the claims requires only the reported certificates and the verifier, rather than reproducing their discovery. The supplementary material provides the software requirements and instructions for running the verification.

\section*{AI Use Statement}
We used large language models (LLMs) as auxiliary tools during the research and preparation of this work. LLMs were used to assist with writing, editing, and optimizing code, primarily to enable rapid exploration of ideas. We also experimented with using LLMs to identify patterns or derive further insights from the solutions found through our computational search; these attempts did not yield any insights or results incorporated into this work. Finally, LLMs were used to proofread the manuscript and provide suggestions for improving its presentation and readability.

The research conception, mathematical claims, and conclusions reported in this paper were developed by the authors and not generated by LLMs. The authors reviewed the AI-assisted material and take full responsibility for the final content of the paper.

\ificlrfinal
\subsubsection*{Acknowledgments}
The authors thank the members of the Neural Polytopes Zulip community and the participants of the Workshop on Polyhedral Geometry for Neural Networks\footnote{\url{https://neuralpolytopes.gitlab.io/workshop2026/}}, held in Nuremberg in March 2026, for inspiring exchange.

\fi

\bibliography{references}
\bibliographystyle{iclr2027/iclr2027_conference}

\clearpage
\appendix
\section{Signed Minkowski identities via linear algebra}
\label{sec:linear-system}

This section gives the polyhedral origin of our search method. The construction of \cite{bakaevBetterNeuralNetwork2026} subdivides the standard simplex and turns the cells into a signed Minkowski identity for \(\maxfn{5}\). We first used exact linear algebra to recover that identity from the cells and their intersections. The same observation then became our search principle: once a family of candidate polytopes is fixed, the unknown signed coefficients satisfy a finite linear system. The main implementation works with the equivalent support-function equations on the sorted cone.

A \emph{polytope} is the convex hull of finitely many points. The support function of a polytope \(P\subset\RR^n\) is
\[
  h_P(u)\coloneqq\max_{p\in P} u^\top p.
\]
Thus, if
\[
  f(x)=\max_{i \in \set{1, \dots, m}} a_i^\top x,
\]
then \(f=h_P\) for the Newton polytope
\(
  P\coloneqq\conv\set{a_1,\ldots,a_m}.
\)
Conversely, the support function of every polytope is a convex, positively homogeneous \(\CPWL\) function. Under this correspondence, addition of support functions corresponds to Minkowski addition, pointwise maxima correspond to convex hulls, and multiplication by a positive scalar corresponds to dilation. Signed linear combinations of support functions are therefore naturally interpreted as formal Minkowski differences.

In particular, \(\maxfn{n}\) is the support function of the standard simplex \(\conv\set{e_1,\ldots,e_n}\), and \(\max\set{x_i,x_j}\) is the support function of the segment \(\segment{e_i}{e_j}\). Hence, the sums of pairwise maxima and the maxima of such sums used in the main text correspond, respectively, to Minkowski sums and convex hulls.

In \cite{bakaevBetterNeuralNetwork2026}, a structured subdivision of the simplex yields a weighted Minkowski difference of polytopes whose support functions have two-hidden-layer representations. Their signed sum is the support function of the simplex and hence equals the maximum function.

Our search begins directly with candidate polytopes \(Q_1,\ldots,Q_r\) whose support functions have two-hidden-layer representations. We seek coefficients \(c_i\in\RR\) such that the formal Minkowski identity
\[
  P = \sum_{i=1}^r c_i Q_i
\]
holds. Some coefficients may be negative. By definition, the identity means that
\[
  P + \sum_{c_i<0} (-c_i)Q_i
  =
  \sum_{c_i > 0} c_i Q_i
\]
as an equality of ordinary Minkowski sums. Equivalently, the support functions satisfy
\[
  h_P=\sum_{i=1}^r c_i h_{Q_i}.
\]

Since the support functions are piecewise linear, we may choose a complete polyhedral fan on whose cones they are all linear. On each cone, their equality is determined by their values on any finite conic generating set. This reduces the desired identity to a finite system of linear equations.

A \emph{complete polyhedral fan} in \(\RR^n\) is a finite collection of polyhedral cones with apex at the origin, closed under taking faces, such that the intersection of any two cones is a face of each and the cones cover \(\RR^n\).

\begin{proposition}
  \label{prop:finite-linear-system}
  Let \(\Sigma\) be a complete polyhedral fan, and let \(P,Q_1,\dots,Q_r\subset\RR^n\) be polytopes whose support functions are linear on every cone of \(\Sigma\). For each \(\tau\in\Sigma\), choose a finite conic generating set \(U_\tau\subset\tau\), so that
  \[
    \tau=\operatorname{cone}(U_\tau).
  \]
  Set
  \[
    U\coloneqq\bigcup_{\tau\in\Sigma}U_\tau,
  \]
  and define
  \[
    A_{u,i}\coloneqq h_{Q_i}(u),
    \qquad
    b_u\coloneqq h_P(u)
    \qquad(u\in U, \; i \in \set{1, \dots, r}).
  \]
  Then a vector \(c=(c_1,\ldots,c_r)\) satisfies
  \[
    h_P=\sum_{i=1}^r c_i h_{Q_i}
  \]
  if and only if \(Ac=b\).
\end{proposition}

\begin{proof}
  Let
  \[
    f\coloneqq h_P-\sum_{i=1}^r c_i h_{Q_i}.
  \]
  The restriction of \(f\) to each cone \(\tau\in\Sigma\) is linear. If \(Ac=b\), then \(f\) vanishes on \(U_\tau\). Every point of \(\tau\) is a nonnegative linear combination of elements of \(U_\tau\), so \(f\) vanishes on all of \(\tau\). Since the cones of \(\Sigma\) cover \(\RR^n\), we obtain \(f\equiv0\). The converse follows by evaluating the functional identity at the points of \(U\).
\end{proof}

\Cref{prop:finite-linear-system} therefore gives a direct search procedure. Let \(P\coloneqq\conv\{e_1,\dots,e_n\}\) be the standard simplex, choose candidate polytopes \(Q_1,\dots,Q_r\) whose support functions have two-hidden-layer representations, and choose a common fan on which all relevant support functions are linear. The proposition reduces the signed Minkowski identity to a rational linear system. A solution gives a two-hidden-layer representation of \(\maxfn{n}\). An inconsistent system rules out only the chosen candidate family.

\begin{remark}
  To recover the \(\maxfn{5}\) construction of~\cite{bakaevBetterNeuralNetwork2026}, take its cells and their nonempty intersections as the candidate summands \(Q_i\). Evaluating their support functions on generators of a common refinement of their normal fans produces \(Ac=b\). A solution recovers the signed coefficients directly from the pieces of the subdivision.
\end{remark}

\section{Structural properties of the ansatz}

\subsection{A necessary lower bound on \texorpdfstring{\(k\)}{k}}
The parameter \(k\) controls the expressivity of the ansatz, as shown in \Cref{lem:span-nested}.
At the same time, the number of pairs of multisets \((A, B) \in \mathcal{M}_{n,k} \times \mathcal{M}_{n,k}\) grows rapidly with \(k\), even after accounting for symmetries. It is therefore natural to ask for the smallest value of \(k\) that is not ruled out by theoretical obstructions. This question also has practical significance, since \(k\) directly governs the size of the system of linear equations we have to solve.

\begin{proposition}
  \label{prop:minkowski-difference-simplex-min-dim}
  Let \(S\subset\RR^N\) be a \(k\)-dimensional simplex. There do not exist polytopes \(P_1,\ldots,P_m\subset\RR^N\) with \(\dim P_r<k\) and coefficients \(c_r\in\RR\) such that
  \[
    h_S=\sum_{r=1}^m c_r h_{P_r}.
  \]
\end{proposition}

\begin{proof}
  Let \(L\coloneqq\operatorname{span}(S-S)\) and let \(\pi\colon\RR^N\to L\) be the orthogonal projection. Restricting the assumed identity to \(L\) gives
  \[
    h_{\pi(S)}=\sum_{r=1}^m c_r h_{\pi(P_r)},
  \]
  since \(h_{\pi(P)}(u)=h_P(u)\) for every \(u\in L\). The polytope \(\pi(S)\) is a \(k\)-dimensional simplex in \(L\), whereas
  \[
    \dim\pi(P_r)\le \dim P_r<k.
  \]
  After identifying \(L\) with \(\RR^k\), separate the positive and negative coefficients in the support-function identity. Additivity of support functions under Minkowski sums and positive dilations gives
  \[
    h_{\pi(S)+\sum_{c_r<0}(-c_r)\pi(P_r)}
    =
    h_{\sum_{c_r>0}c_r\pi(P_r)}.
  \]
  Since a compact convex set is determined by its support function, it follows that
  \[
    \pi(S)+\sum_{c_r<0}(-c_r)\pi(P_r)
    =
    \sum_{c_r>0}c_r\pi(P_r).
  \]
  Every polytope appearing in the two sums has dimension strictly less than \(k\), and hence has zero volume in \(L\). Thus \(\pi(S)\) is a zero summand, contradicting \cite[Corollary~5.2]{koutschanRepresentingPiecewiseLinear2025}.
\end{proof}

\begin{corollary}
  \label{cor:min-k}
  The exact systems assembled from the ansatz above can have a solution only if
  \[
    k \ge k_{\min}
    \coloneqq \floor*{\frac{n-1}{2}}.
  \]
  Thus, \(k_{\min}\) is the smallest value of \(k\) not ruled out by \Cref{prop:minkowski-difference-simplex-min-dim}.
\end{corollary}

\begin{proof}
  For each \((i,j)\), \(m_{ij}=h_{\segment{e_i}{e_j}}\). Hence
  \[
    \sum_{(i,j)\in A} m_{ij}=h_{Z_A},
    \qquad \text{where}\qquad
    Z_A \coloneqq \sum_{(i,j)\in A}[e_i,e_j],
  \]
  and \(\dim Z_A\le k\). Therefore
  \[
    \Phi_{A,B}
    =\max\set{h_{Z_A},h_{Z_B}}
    =h_{\conv(Z_A\cup Z_B)}.
  \]
  Since \(Z_A\) and \(Z_B\) both have dimension at most \(k\), their convex hull has dimension at most \(2k+1\). Consequently, after expanding the symmetrized blocks into their summands, any solution of the exact system would express
  \[
    \max(x_1,\ldots,x_n)
    =h_{\conv\set{e_1,\ldots,e_n}}
  \]
  as a signed linear combination of support functions of polytopes of dimension at most \(2k+1\). The polytope \(\conv\set{e_1,\ldots,e_n}\) is an \((n-1)\)-dimensional simplex. Hence, if \(2k+1<n-1\), this representation contradicts \Cref{prop:minkowski-difference-simplex-min-dim}. We must therefore have
  \[
    2k+1\ge n-1,
  \]
  which is equivalent to \(k\ge\floor{(n-1)/2}\).
\end{proof}

\subsection{Nestedness of the search spaces}
The parameter \(k\) does control the complexity of our ansatz, in the fallowing we prove, that increasing \(k\) does not decrease the expressive power of the ansatz, i.e. if there is a solution for some \(k_0\) then also for all \(k > k_0\).

\begin{lemma}
  \label{lem:span-nested}
  Fix \(n\). The spaces \(V_{n,k}\) defined in \Cref{sec:ansatz} are nested: \(V_{n,k} \subseteq V_{n,k+1}\).
\end{lemma}

\begin{proof}
  It suffices to show that every generator of \(V_{n,k}\) belongs to \(V_{n,k+1}\). Fix \(A,B \in \mathcal{M}_{n,k}\) and choose any index pair \(p=(i,j) \in \mathcal{E}_n\). Let \(\uplus\) denote multiset union, i.e., addition of multiplicities, and define
  \[
    A^+ \coloneqq A \uplus \set{p},
    \qquad
    B^+ \coloneqq B \uplus \set{p},
  \]
  so that \(A^+,B^+ \in \mathcal{M}_{n,k+1}\). Adding the same term \(m_p\coloneqq\max(x_i,x_j)\) to both multisets gives
  \[
    \Phi_{A^+,B^+}=\Phi_{A,B}+m_p.
  \]
  Symmetrizing both sides yields, for all $x\in \RR^n$,
  \[
    F_{A^+,B^+}(x)
    =
    F_{A,B}(x)+\sum_{\sigma \in \symgrp{n}} m_p(\sigma x).
  \]

  Now let \(C\) be the multiset consisting of \(k+1\) copies of \(p\). Since \(C \in \mathcal{M}_{n,k+1}\),
  \[
    F_{C,C}=(k+1)\sum_{\sigma \in \symgrp{n}} m_p(\sigma x).
  \]
  Therefore,
  \[
    F_{A,B}=F_{A^+,B^+}-\frac{1}{k+1}F_{C,C}.
  \]
  Both terms on the right-hand side belong to \(V_{n,k+1}\), and hence \(F_{A,B} \in V_{n,k+1}\). Because \(F_{A,B}\) was an arbitrary generator of \(V_{n,k}\), it follows that \(V_{n,k} \subseteq V_{n,k+1}\).
\end{proof}

\section{Width of the construction}
\label{sec:construction-width}

The constructions above are designed for a tractable search for two-hidden-layer representations; width is not optimized. Suppose that the certificate has support size \(s\). Expanding each supported function \(F_{A,B}\) into its permutation sum produces at most \(n!\) labelled blocks. The second hidden layer therefore contains at most
\[
  Q\leq s n!
\]
rank-\(2\) maxout units. By sharing the required pairwise maxima, the first hidden layer has width at most
\[
  \binom{n+1}{2}
  =
  \bigO(n^2)
\]
and the total number of hidden maxout units is at most
\[
  \binom{n+1}{2}+Q.
\]
For the corresponding standard ReLU network without skip connections, the first hidden layer has width at most
\[
  \binom{n}{2}+2n
  =
  \bigO(n^2)
\]
and second-layer width at most \(3Q\leq 3s n!\). Thus the direct symmetrized realization has quadratic width in the first hidden layer and, in the worst case, factorial width in the second. We did not attempt to optimize this crude upper bound, which remains far above the quadratic lower bounds for two-hidden-layer representations established by \cite{safranDepthHierarchyComputing2026}.

\end{document}